\documentclass[preprint,12pt]{elsarticle}
\usepackage{verbatim}
\usepackage{amssymb}
\usepackage{graphicx}
\usepackage{cite}
\usepackage{color}
\usepackage{subfigure}
\usepackage{graphicx}
\usepackage{amsmath}
\usepackage[ruled]{algorithm2e}
\usepackage{amsthm}

\begin{document}

\title{Gene Selection With  \\ Guided Regularized Random Forest}

\author{Houtao Deng}
\ead{hdeng3@asu.com}
\address{Intuit, Mountain View, CA, USA}
\author{George Runger}
\ead{george.runger@asu.edu}
\address{Arizona State University, Tempe, AZ, USA}
\cortext[cor1]{This research was partially supported by ONR grant N00014-09-1-0656.}


\begin{abstract}
  The regularized random forest (RRF) was recently proposed for feature selection by building only one ensemble. In RRF the features are evaluated on a part of the training data at each tree node. We derive an upper bound for the number of distinct Gini information gain values in a node, and show that many features can share the same information gain at a node with a small number of instances and a large number of features. Therefore, in a node with a small number of instances, RRF is likely to select a feature not strongly relevant.

  Here an enhanced RRF, referred to as the guided RRF (GRRF), is proposed. In GRRF, the importance scores from an ordinary random forest (RF) are used to guide the feature selection process in RRF. Experiments on 10 gene data sets show that the accuracy performance of GRRF
  is, in general, more robust than RRF when their parameters change. GRRF is computationally efficient, can select compact feature subsets, and has competitive accuracy performance, compared to RRF, varSelRF and LASSO logistic regression (with evaluations from an RF classifier). Also, RF applied to the features selected by RRF with the minimal regularization outperforms RF applied to all the features for most of the data sets considered here. Therefore, if accuracy is considered more important than the size of the feature subset, RRF with the minimal regularization may be considered. We use the accuracy performance of RF, a strong classifier, to evaluate feature selection methods, and illustrate that weak classifiers are less capable of capturing the information contained in a feature subset. Both RRF and GRRF were implemented in the ``RRF" R package available at CRAN, the official R package archive.
\end{abstract}


\maketitle


\newtheorem{Lemma}{Lemma}
\newtheorem{Theorem}{Theorem}

\section{Introduction}
Given a training data set consisting of $N$ instances, $P$ predictor variables/features ${X_i (i=1,...,P})$ and the class $Y\in\{1,2,...,C\}$, the objective of feature selection is to select a compact variable/feature subset without loss of predictive information about $Y$. Note feature selection selects a subset of the original feature set, and, therefore, may be more interpretable than feature extraction (e.g., principal component analysis \citep{jolliffe2002principal} and partial least squares regression \citep{geladi1986partial}) which creates new features based on transformations or combinations of the original feature set \citep{jain1997feature}. Feature selection has been widely used in many applications such as gene selection \citep{ruiz2006incremental,zhu2007markov,liu2010ensemble} as it can moderate the curse of dimensionality, improve interpretability \citep{guyon2003} and avoid the effort to analyze irrelevant or redundant features.


Information-theoretic measures such as symmetrical uncertainty and mutual information can measure the degree of association between a pair of variables and have been successfully used for feature selection, e.g., CFS (correlation-based feature selection) \citep{hall2000} and FCBF (fast correlation-based filter) \citep{liuhuan2004}. However, these measures are limited to two variables and do not capture high-order interactions between variables well. For example, the measures can not capture the exclusive OR relationship $Y$ = $XOR(X_1, X_2)$, in which neither $X_1$ nor $X_2$ is predictive individually, but $X_1$ and $X_2$ together can correctly determine $Y$ \citep{jakulin2003analyzing}.

LASSO logistic regression \citep{tibshirani1996regression} and recursive feature elimination with a linear SVM (SVM-RFE) \citep{guyon2002} are well-known feature selection methods based on classifiers. These methods assume a linear relationship between the log odds of the class and the predictor variables (LASSO logistic regression) or between the class and the predictor variables (linear SVM). Furthermore, before using these methods, one often needs to preprocess the data such as transforming categorical variables to binary variables or normalizing variables of different scales.

A random forest (RF)\citep{breiman2001} classifier has been commonly used for measuring feature importance \citep{Andy2002}. An RF naturally handles numerical and categorical variables, different scales, interactions and nonlinearities, etc. Although the RF feature importance scores can be used to select $K$ features with the highest importance scores \emph{individually}, there could be redundancy among the $K$ features. Consequently, the selected features subset can differ from the best \emph{combination} of $K$-features, that is, the best feature subset. Similarly, Boruta \citep{rudnicki2010feature}, a method based on RF, aims to select a set of relevant features, which is different from the objective of a relevant and also non-redundant feature subset.

A feature selection method based on RF, varSelRF\citep{Uriarte2006}, has become popular. varSelRF consists of multiple iterations and eliminates the feature(s) with the least importance score(s) at each iteration. Since eliminating one feature at each iteration is computationally expensive, the authors considered eliminating a fraction, e.g., $1/5$, of the features at each iteration. However, when there is a large number of features, many features are eliminated at each iteration, and, thus, useful features, but with small importance scores, can be eliminated.

The ACE algorithm \citep{tuv2009} is another ensembles-based feature selection method. It was shown to be effective, but it is more computationally demanding than the simpler approaches considered here. It requires multiple forests to be constructed, along with multiple gradient boosted trees \citep{friedman01greedy}.

Recently, the regularized random forest (RRF) was proposed for feature selection with one ensemble \citep{houtaoFSReg2012}, instead of multiple ensembles \citep{Uriarte2006,tuv2009}. However, in RRF the features are evaluated on a part of the training data at each tree node and the feature selection process may be greedy.

Here we analyze a feature evaluation issue that occurs in all the usual splitting algorithms at tree nodes with a small number of training instances. To solve this issue, we propose the guided RRF (GRRF) method, in which the importance scores from an ordinary RF are used to guide the feature selection process in RRF. Since the importance scores from an RF are aggregated from all the trees based on all the training data, GRRF is expected to perform better than RRF. 


Section \ref{sec:Background} presents previous work. Section \ref{sec:IssueSparse} discusses the node sparsity issue when evaluating features at tree nodes with a small number of training instances. Section \ref{sec:GRRF} describes the GRRF method. 
Section \ref{sec:EXP} presents and discusses the experimental results. Section \ref{sec:CON} concludes this work.

\section{Background}\label{sec:Background}
\subsection{Variable importance scores from Random Forest}\label{sec:RF}
A random forest (RF) \citep{breiman2001} is a supervised learner that consists of multiple decision trees, each of which grown on a bootstrap sample from the original training data. The Gini index at node $v$, $Gini(v)$, is defined as
$$
Gini(v)= \sum_{c=1}^{C}\hat p^v_c (1 - \hat p^v_c)
$$
where $\hat p^v_c$ is the proportion of class-$c$ observations at node $v$. The Gini information gain of $X_i$ for splitting node $v$, $Gain(X_i,v)$, is the difference between the impurity at the node $v$ and the weighted average of impurities at each child node of $v$. That is, 
$$\setlength\arraycolsep{0.1em}
\begin{array}{ll}
  Gain(X_i,v) = &  \nonumber\\
  Gini(X_i,v)-w_L Gini(X_i,v^L)-w_R Gini(X_i,v^R) &   \nonumber\\
\end{array}
$$
where $v^L$ and $v^R$ are the left and right child nodes of $v$, respectively, and $w_L$ and $w_R$ are the proportions of instances assigned to the left and right child nodes. At each node, a random set of $mtry$ features out of $P$ is evaluated, and the feature with the maximum $Gain(X_i,v)$ is used for splitting the node $v$.

The importance score for variable $X_i$ can be calculated as
\begin{equation}\label{equation:RFImp}
Imp_i=\frac{1}{ntree}\sum_{v\in S_{X_i}}Gain(X_i,v)
\end{equation}
where $S_{X_i}$ is the set of nodes split by $X_i$ in the RF with $ntree$ trees. The RF importance scores are commonly used to evaluate the contributions of the features regarding predicting the classes. 

\IncMargin{1em}
\begin{algorithm}
\scriptsize
\LinesNumbered
\SetKwData{Left}{left}\SetKwData{This}{this}\SetKwData{Up}{up}
\SetKwFunction{Union}{Union}\SetKwFunction{FindCompress}{FindCompress}
\SetKwInOut{Input}{input}\SetKwInOut{Output}{output}
\Input{$F$ and $\lambda$}
\Output{$F$}
\BlankLine
\emph{$gain^* \leftarrow 0$, $count \leftarrow 0$, $f^* \leftarrow -1$, $f_1 \leftarrow \{1,2,3,...,P\} $}, $f_2 \leftarrow \emptyset$ \ \\
\While{$f_1 \neq \emptyset$}{
$m \leftarrow select(f_1)$ //randomly select feature index $m$ from $f_1$ \\
$f_2 \leftarrow \{f_2,m\}$ //add $m$ to  $f_2$\\
$f_1 \leftarrow f_1 - m$   //remove $m$ from $f_1$\\
\emph{$gain_R(X_m,v) \leftarrow 0$}\

\If{$X_m\in F$}{\label{lt}
$gain_R(X_m,v)\leftarrow gain(X_m,v)$ //calculate $gain_R$ for all variables in $F$
}
\If{$X_m \notin F$ and $count<\lceil \sqrt P \rceil$}{\label{lt}
$gain_R(X_m) \leftarrow \lambda \cdot gain(X_m)$ \ //regularize the gain if the variable is not in $F$
$count \leftarrow count+1$\
}
\If{$gain_R(X_m,v) > gain^*$}{\label{lt}
$gain^* \leftarrow gain_R(X_m,v)$, $f^* \leftarrow m$\
}
}
\If{$f^* \neq -1$ and $f^* \notin F$}{\label{lt}
$F\leftarrow\{F,f^*\}$\
}
\Return{F}
\caption{Feature selection at node $v$.}\label{algo_disjdecomp}
\end{algorithm}\DecMargin{1em}

\subsection{Regularized Random Forest}\label{sec:RRF}
The regularized random forest (RRF) \citep{houtaoFSReg2012} applies the tree regularization framework to RF and can select a compact feature subset. While RRF is built in a way similar to RF, the main difference is that the regularized information gain, $Gain_R(X_i,v)$, is used in RRF
\begin{equation}\label{eq:regularizedIG}
Gain_R(X_i,v) =
\begin{cases}
 \lambda \cdot Gain(X_i,v) & \text{$i \notin F$} \\
 Gain(X_i,v) & \text{$i \in F$} \\
\end{cases}
\end{equation}
where $F$ is the set of indices of features used for splitting in previous nodes and is an empty set at the root node in the first tree. Here $\lambda \in(0,1]$ is called the penalty coefficient. When $i \notin F$, the coefficient penalizes the $i^{th}$ feature for splitting node $v$. A smaller $\lambda$ leads to a larger penalty. RRF uses $Gain_R(X_i,v)$ at each node, and adds the index of a new feature to $F$ if the feature adds enough predictive information to the selected features.

RRF with $\lambda=1$, referred to as RRF(1), has the minimum regularization. Still, a new feature has to be more informative at a node than the features already selected to enter the feature subset. The feature subset selected by RRF(1) is called the \emph{least regularized subset}, indicating the minimal regularization from RRF.

Figure \ref{fig:fsProcess} illustrates the feature selection process. The nodes in the forest are visited sequentially (from the left to the right and from the top to the bottom). The indices of three distinct features used for splitting are added to $F$ in Tree 1, and
the index of $X_5$ is added to $F$ in Tree 2. Algorithm \ref{algo_disjdecomp} shows the feature selection process at each node.

\def\myWidth{3.5}
\begin{figure*}[ht]
\centering
\includegraphics[width= \myWidth in]{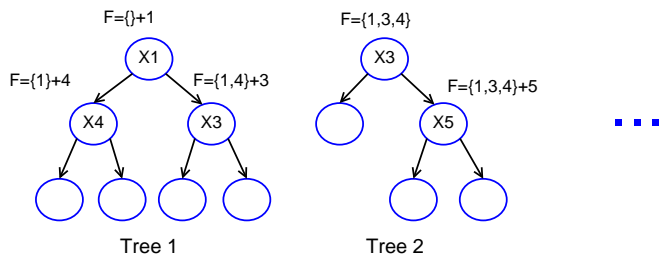}
\caption[2]{The feature selection procedure of RRF. The non-leaf nodes are marked with the splitting variables. Three distinct features used for splitting are added to $F$ in Tree 1, and one feature $X_5$ is added to $F$ in Tree 2.\label{fig:fsProcess}}
\end{figure*}

Similar to $Gain_R(\cdot)$, a penalized information gain was used to suppress spurious interaction effects in the rules extracted from tree models \citep{friedmanpredictive2008}. The objective of \citet{friedmanpredictive2008} was different from the goal of a compact feature subset here. Also, the regularization used in \citet{friedmanpredictive2008} only reduces the redundancy in each path from the root node to a leaf node, but the features extracted from such tree models can still be redundant.

\section{The Node Sparsity Issue} \label{sec:IssueSparse}
This section discusses an issue of evaluating features in a tree node with a small number of instances, referred to as the \emph{node sparsity issue} here. The number of instances decreases as the instances are split recursively in a tree, and, therefore, this issue can commonly occur in tree-based models.

In a tree node, the Gini information gain or other information-theoretic measures are commonly used to evaluate and compare different features. However, the measures calculated from a node with a small number of instances may not be able to distinguish the features with different predictive information. In the following we establish an upper bound for the number of distinct values of Gini information gain for binary classification.

Let $D(f)$ denote the number of distinct values of a function $f$ (defined over a specified range), $N$ denote the number of instances at a non-leaf node $v$ and $N\geq2$ (otherwise it cannot be split). For simplicity, in the following we also assume $N$ is even. The following procedure is similar when $N$ is odd. Let $N_1$ and $N_2$ be the number of instances of class 1 and class 2, respectively, at the node. Let $L$ denote the number of instances at the left child node, and let $L_1$ and $L_2$ denote the number of instances for class 1 and class 2 at the left child node, respectively, with similar notations $R$, $R_1$ and $R_2$ for the right child node. Note $L\geq1$ and $R\geq1$. 
The notation $L_1L_2$ denotes the product function where $L_1$ and $L_2$ assume values in the feasible domain $0 \le L_1, L_2 \le L$ and $L_1 + L_2=L$.
Then we have the following lemmas and theorem.

\begin{Lemma}\label{Lemma1}
An upper bound of the number of distinct values of $L_1L_2$ is $\lceil(L+1)/2\rceil$. That is, $D(L_1L_2) \leq \lceil(L+1)/2\rceil$, where $\lceil \cdot \rceil$ denotes the ceiling.
\end{Lemma}
\begin{proof}
$L_1$ has at most $L+1$ values ($L_1\in$\{0,1,...,L\}). Now let $l_1l_2$ and $m_1m_2$ be two realizations of $L_1L_2$, and let $l_1l_2=m_1m_2$, we have
$$\setlength\arraycolsep{0.1em}
\begin{array}{ccccr}
  &(L-l_1)l_1 & = & (L-m_1)m_1 \nonumber\\
   \Leftrightarrow& Ll_1 - l_1^2 - Lm_1 + m_1^2       & = & 0  \nonumber\\
   \Leftrightarrow& L(l_1-m_1)+(m_1+l_1)(m_1-l_1)       & = & 0  \nonumber\\
   \Leftrightarrow& (l_1 - m_1)(L - l_1 - m_1)      & = & 0  \nonumber\\
   \Leftrightarrow& l_1 = m_1 \ or \ l_1 = L - m_1     &  &   \nonumber\\
\end{array}
$$
Therefore, we obtain an upper bound for the number of distinct values for $L_1L_2$ as $D(L_1L_2)<= \lceil(L+1)/2\rceil$.
\end{proof}


\begin{Lemma}\label{Lemma2}
$D(L_1L_2/L) \le N(N+2)/4 - 1$.
\end{Lemma}
\begin{proof}
For each $L$, $L_1L_2/L$ has at most $D(L_1L_2)$ distinct values. Because $1 \leq L \leq N-1$, an upper bound for $D(L_1L_2/L)$ is derived as follows
\begin{equation}
  D(L_1L_2/L)  \le  \lceil2/2\rceil+\lceil3/2\rceil+...+\lceil N/2\rceil  =  N(N+2)/4-1  \nonumber\\
\end{equation}
\end{proof}

\begin{Lemma}\label{Lemma3}
$D(L_1L_2/L + R_1R_2/R) \le N(N+2)/4 - 1$
\end{Lemma}
\begin{proof}
Because $L_1L_2/L$ and $R_1R_2/R$ are symmetric, the upper bound is the same as for one term. That is, $D(L_1L_2/L + R_1R_2/R) \le N(N+2)/4 - 1$
\end{proof}

\begin{Theorem}\label{Theory1}
For binary classification, an upper bound of the number of distinct information gain values is $N(N+2)/4 - 1$.
\end{Theorem}
\begin{proof}
The information gain for spliting node $v$ is
$$Gain(v) = Gini(v)-w_L Gini(v^L)-w_R Gini(v^R)$$
Because $Gini(v)$ is a constant at node $v$, we only need to consider $D(w_L Gini(v^L)+w_R Gini(v^R))$. For two classes
$$\setlength\arraycolsep{0.1em}
\begin{array}{ll}
    & w_L Gini(v^L)  \nonumber\\
=   & (L/N)(L_1/L \cdot L_2/L + L_2/L \cdot L_1/L) \nonumber\\
=   & 2(L/N)(L_1L_2)/L^2
\end{array}
$$
Then
$$\setlength\arraycolsep{0.1em}
\begin{array}{ll}
    & w_L Gini(v^L)+w_R Gini(v^R) \nonumber\\
=   & 2(L/N)(L_1L_2)/L^2+ 2(R/N)(R_1R_2)/R^2  \nonumber\\
=   & (2/N)(L_1L_2/L) + (2/N)(R_1R_2/R)
\end{array}
$$
Because $N$ is a constant, we have $D(w_L Gini(v^L)+w_R Gini(v^R)) = D(L_1L_2/L+R_1R_2/R)$.
According to Lemma \ref{Lemma3}, $D(L_1L2/L + R_1R_2/R) \le N(N+2)/4 - 1$.
\end{proof}
Note similar conclusions may be applied to other information-theoretic measures such as the regularized information gain in Equation \ref{eq:regularizedIG}.

Consequently, when $N$ is small, the number of distinct Gini information gain values is small. For a large number of variables, many could have the same Gini information gain. For example, at a node with 10 instances there are at most 29 distinct Gini information gain values for binary classification problems. For 1000 genes with two classes, there are at least 1000-29=971 genes having the same information gain as other genes. The number of instances can be even smaller than 10 in a node of an RF or RRF as each tree is grown completely.

In RRF, a feature with the maximum regularized information gain is added at a node based on only the instances at that node. When multiple features have the maximum regularized information gain, one of these features is randomly selected. As discussed above, at a node with only a small number of instances less relevant, or redundant features may be selected. An additional metric is useful to help distinguish the features. In the following section, we introduce the guided regularized random forest that leverages the importance scores calculated from an RF based on all the training data.

\section{Guided Regularized Random Forest}\label{sec:GRRF}
The guided RRF (GRRF) uses the importance scores from a preliminary RF to guide the feature selection process of RRF. Because the importance scores from an RF are aggregated from all the trees of the RF based on all the training data, GRRF may be able to handle the node sparsity issue discussed in the previous section.

A normalized importance score is defined as
\begin{equation}
Imp_i' = \frac{Imp_i}{max_{j=1}^{P}Imp_j}
\end{equation}
where $Imp_i$ is the importance score from an RF (defined in Equation \ref{equation:RFImp}). Here $0\leq Imp_i'\leq 1$. Instead of assigning the same penalty coefficient to all features in RRF, GRRF assigns a penalty coefficient to each feature. That is,
\begin{equation}
Gain_R(X_i,v) =
\begin{cases}
 \lambda_i Gain(X_i,v) & \text{$X_i \notin F$} \\
 Gain(X_i,v) & \text{$X_i \in F$} \\
\end{cases}
\end{equation}
where $\lambda_i \in(0,1]$ is the coefficient for $X_i$ ($i\in\{1,...,P\}$) and is calculated based on the importance score of $X_i$ from an ordinary RF. That is,
\begin{equation}
 \lambda_i = (1-\gamma)\lambda_0 + \gamma Imp_i'\
\end{equation}
where $\lambda_0\in(0,1]$ controls the degree of regularization and is called the base coefficient, and $\gamma\in[0,1]$ controls the weight of the normalized importance score and is called the importance coefficient. Note that RRF is a special case of GRRF with $\gamma=0$. Given $\lambda_0$ and $\gamma$, a feature with a larger importance score has a larger $\lambda_i$, and, therefore, is penalized less. 

In our experiments we found the feature subset size can be effectively controlled by changing either $\lambda_0$ or $\gamma$, but changing the latter often leads to better performance in terms of classification accuracy. To reduce the number of parameters of GRRF, we fix $\lambda_0$ to be 1 and consider $\gamma$ as the only parameter for GRRF. With $\lambda_0=1$, we have
\begin{equation}
\lambda_i = (1-\gamma) + \gamma Imp_i' = 1- \gamma(1-Imp_i')
\end{equation}
For a feature $X_i$ that does not have the maximum importance score ($Imp_i'\neq1$), a larger $\gamma$ leads to a smaller $\lambda_i$ and, thus, a larger penalty on $Gain(X_i,v)$ when $X_i$ has not been used in the nodes prior to node $v$. Consequently, $\gamma$ is essentially the degree of regularization.
Furthermore, GRRF with $\gamma=0$ is equivalent to RRF(1), with the minimal regularization.

\section{Experiments}\label{sec:EXP}

We implemented the RRF and GRRF algorithms in the ``RRF" R package based on the ``randomForest" package \citep{Andy2002}. The ``RRF" package is available from CRAN (http://cran.r-project.org/), the official R package archive.

Work by \citep{Uriarte2006} showed that the accuracy performance of RF is largely independent of the number of trees (between 1000 and 40000 trees), and $mtry=\sqrt P$ is often a reasonable choice. Therefore, we used $ntree=1000$ and $mtry=\sqrt P$ for the classifier RF and the feature selection methods RRF, GRRF and varSelRF. Guided by initial experiments, randomly sampling 63\% of the data instances (the same as the default setting in the ``randomForest" package) without replacement was used in RRF and GRRF. We also evaluated two well-known feature selection methods: varSelRF available in the ``varSelRF" R package and LASSO logistic regression available in the ``glmnet" R package. Unless otherwise specified, the default values in the packages were used. Also note $\lambda$ is the parameter of RRF, and as discussed, $\gamma$ is considered as the only parameter of GRRF here.

\begin{table*}[h]
\scriptsize
\centering
\caption{The number of groups identified, and the number of irrelevant or redundant features selected for different algorithms. \label{table:simuErr}}
\begin{tabular}{|c|c|c|c|c|c|c|c|c|}
\hline
           &                \multicolumn{ 4}{|c|}{\# Groups } & \multicolumn{ 4}{|c|}{\# Irrelevant or Redundant Features} \\
\hline
    Method &      LASSO &   varSelRF &        RRF &       GRRF &      LASSO &   varSelRF &        RRF &       GRRF \\
\hline
   Average &          4 &       4.95 &          5 &       4.95 &        4.3 &       4.95 &       4.95 &       0.75 \\
\hline
  Std. err &      0.000 &      0.050 &      0.000 &      0.050 &      0.147 &      0.050 &      0.050 &      0.123 \\
\hline
\end{tabular}

\end{table*}

\subsection{Simulated Data Sets}
We start with a simulated data set generated by the following procedure. First generate 10 independent variables: $X_1$,..., $X_{10}$, each is uniformly distributed in the interval [0,1]. Variable $Y$ is calculated by the formula
$$Y = 10\sin(\pi X_1X_2) + 20(X_3 - 0.5)^2 + 10X_4 + 5X_5 + e$$
where $e$ follows a standard normal distribution. Note the above data generation procedure was described in \citep{friedman1991multivariate}. We simulated 1000 instances and then calculated the median of $Y$ as $\overline{y}$. We then labeled class 2 to the instances with $Y>\overline{y}$, and class 1 otherwise, so that it becomes a classification problem. Furthermore, we added five other variables with $X_{11}$=$X_1$, $X_{12}$=$X_2$, $X_{13}$=$X_3$, $X_{14}$=$X_4$, $X_{15}$=$X_5$. Consequently, the feature selection solution is \{$(X_1|X_{11})$ \& $(X_2|X_{12})$ \& $(X_3|X_{13})$ \& $(X_4|X_{14})$ \& $(X_5|X_{15})$\}, where \& stands for ``and" and $(X_1|X_{11})$ indicates that one and only one from this group should be selected. For example, for the data set considered here, both $(X_1,X_2,X_3,X_4,X_5)$ and $(X_{11},X_2,X_{13},X_4,X_5)$ are correct solutions. However, $(X_1,X_3,X_4,X_5)$ misses the group $(X_2|X_{12})$. Furthermore, $(X_1,X_2,X_3,X_4,X_5,X_{11})$ has a redundant variable because $X_{11}$ is redundant to $X_1$.

We simulated 20 replicates of the above data set, and then applied GRRF, RRF and two well-known methods varSelRF and LASSO logistic regression to the data sets. Here $\gamma$ of GRRF was selected from \{0.4,0.5,0.6\}, $\lambda$ of RRF was selected from \{0.6,0.7,0.8\}, and the regularization parameter of LASSO logistic regression was selected from \{0.01,0.02,0.03,0.05,0.1,0.2\}, all by 10-fold CV.

The results are shown in Table \ref{table:simuErr}. LASSO logistic regression identifies the least number of groups on average (4). The methods varSelRF, RRF and GRRF identify almost all the groups, but varSelRF and RRF select 4.95, on average, irrelevant or redundant variables, and GRRF only selects 0.75 irrelevant or redundant variables. This experiment shows GRRF's potential in selecting a relevant and non-redundant feature subset.

\begin{table*}[!]
\centering
\scriptsize
\caption{Summary of the data sets.}
\label{table:summary}
\begin{tabular}{ccccc}
\hline
   Data set &     Reference &      \# Examples &   \# Features &    \# classes \\
\hline
adenocarcinoma & \citep{Adenocarcinoma2003} &         76 &       9868 &          2 \\

     brain & \citep{Brain2002} &         42 &       5597 &          5 \\

breast.2.class & \citep{Breast2002} &         77 &       4869 &          2 \\

breast.3.class & \citep{Breast2002} &         95 &       4869 &          3 \\

     colon & \citep{Colon1999} &         62 &       2000 &          2 \\

  leukemia & \citep{Leukaemia1999} &         38 &       3051 &          2 \\

  lymphoma & \citep{Lymphoma2000} &         62 &       4026 &          3 \\

     nci60 & \citep{NCI602000} &         61 &       5244 &          8 \\

  prostate & \citep{Prostate2002} &        102 &       6033 &          2 \\

     srbct & \citep{Srbct2001} &         63 &       2308 &          4 \\
\hline
\end{tabular}
\end{table*}

\subsection{Gene Data Sets}
The 10 gene expression data sets analyzed by \citet{Uriarte2006} are considered in this section. The data sets are summarized in Table \ref{table:summary}. For each data set, a feature selection algorithm was applied to 2/3 of the instances (selected randomly) to select a feature subset. Then a classifier was applied to the feature subset. The error rate is obtained by applying the classifier to the other 1/3 of the instances. This procedure was conducted 100 times with different random seeds, and the average size of feature subsets and the average error rate over 100 runs were calculated. In the experiments, we considered RRF, GRRF, varSelRF and LASSO logistic regression as the feature selection algorithms, and random forest (RF) and C4.5 as the classifiers. 

\begin{table}[!]
\scriptsize
\centering
\caption{The number of features selected in the least regularized subset selected by RRF (i.e., RRF(1)), the total number of original features (``All"), and the error rates of RF applied to the least regularized subset and all the features. The win-lose-tie results of ``All" compared to RRF(1) are shown. Here ``$\circ$" or ``$\bullet$" represents a significant difference at the 0.05 level, according to the paired t-test. 
RRF(1) uses many fewer features than the original features, and wins on 7 data sets. There are significant differences for four data sets.  \label{table:RFGRRF}}
\begin{tabular}{c|cc|ccc}
\hline
           & \multicolumn{ 2}{|c}{Number of features} & \multicolumn{ 2}{|c}{Average error rates} &            \\

           &    RRF(1) &        All &    RRF(1)-RF &        All-RF &            \\
\hline
adenocarcinoma &         86 &       9868 &      0.158 &      0.159 &            \\

     brain &         97 &       5597 &      0.159 &      0.170 &            \\

breast.2.class &        210 &       4869 &      0.352 &      0.371 &    $\circ$ \\

breast.3.class &        253 &       4869 &      0.397 &      0.415 &    $\circ$ \\

     colon &         92 &       2000 &      0.162 &      0.158 &            \\

  leukemia &         24 &       3051 &      0.053 &      0.064 &            \\

  lymphoma &         31 &       4026 &      0.018 &      0.006 &   $\bullet$ \\

       nci &        197 &       5244 &      0.332 &      0.321 &            \\

  prostate &         88 &       6033 &      0.082 &      0.109 &    $\circ$ \\

     srbct &         51 &       2308 &      0.027 &      0.032 &            \\
\hline
win-lose-tie &          - &          - &          - &      3-7-0 &            \\
\hline
\end{tabular}
\end{table}

\begin{table}[!]
\scriptsize
\centering
\caption{The total number of original features (``All") and the average number of features (from 100 replicates for each data set) selected by different methods. All feature selection methods are able to select a small number of features. Here GRRF(0.1) and RRF(0.9) select a similar number of features, but it is shown later that GRRF(0.1) is more accurate than RRF(0.9) (with an RF classifier) for most data sets. \label{table:aveFeaN}}
\begin{tabular}{ccccccc}
\hline
    & All & GRRF(0.1) & GRRF(0.2) & RRF(0.9) & varSelRF & LASSO \\
\hline
adenocarcinoma & 9868 &  20 &  15 &  23 &   4 &   2 \\

brain & 5597 &  22 &  12 &  27 &  28 &  24 \\

breast.2.class & 4869 &  59 &  25 &  60 &   7 &   7 \\

breast.3.class & 4869 &  77 &  31 &  78 &  12 &  18 \\

colon & 2000 &  29 &  13 &  27 &   4 &   7 \\

leukemia & 3051 &   6 &   4 &   6 &   2 &   8 \\

lymphoma & 4026 &   5 &   4 &   4 &  81 &  25 \\

nci & 5244 &  63 &  26 &  61 &  60 &  53 \\

prostate & 6033 &  18 &  13 &  19 &   6 &  12 \\

srbct & 2308 &  13 &   9 &  14 &  34 &  28 \\
\hline
\end{tabular}
\end{table}

\begin{table}[!]
\scriptsize
\centering
\caption{The error rates of RF applied to the feature subsets selected by GRRF(0.1), GRRF(0.2), RRF(0.9), varSelRF and LASSO logistic regression, respectively (to three decimal places). The win-lose-tie results of each competitor compared to GRRF(0.1) with RF are shown. Here ``$\circ$" or ``$\bullet$" represents a significant difference between a method and GRRF(0.1) with RF at the 0.05 level, according to the paired t-test.
Here GRRF(0.1) leads to competitive accuracy performance, compared to GRRF(0.2), RRF(0.9), LASSO and varSelRF. \label{table:aveErrRF}}
\begin{tabular}{cccccccccc}
\hline
           &  GRRF(0.1) &  GRRF(0.2) &            &   RRF(0.9) &            &   varSelRF &            &      LASSO &            \\
           &  -RF &  -RF &            &   -RF &            &   -RF &            &      -RF &            \\
\hline
adenocarcinoma &      0.169 &      0.168 &            &      0.160 &            &      0.212 &    $\circ$ &      0.189 &    $\circ$ \\

     brain &      0.214 &      0.259 &    $\circ$ &      0.234 &            &      0.231 &            &      0.259 &    $\circ$ \\

breast.2.class &      0.345 &      0.359 &    $\circ$ &      0.367 &    $\circ$ &      0.386 &    $\circ$ &      0.366 &    $\circ$ \\

breast.3.class &      0.387 &      0.403 &    $\circ$ &      0.410 &    $\circ$ &      0.418 &    $\circ$ &      0.400 &            \\

     colon &      0.175 &      0.186 &    $\circ$ &      0.190 &    $\circ$ &      0.232 &    $\circ$ &      0.180 &            \\

  leukemia &      0.080 &      0.093 &            &      0.091 &            &      0.107 &    $\circ$ &      0.076 &            \\

  lymphoma &      0.067 &      0.076 &            &      0.098 &    $\circ$ &      0.022 &   $\bullet$ &      0.009 &   $\bullet$ \\

       nci &      0.389 &      0.452 &    $\circ$ &      0.405 &            &      0.418 &    $\circ$ &      0.396 &            \\

  prostate &      0.085 &      0.085 &            &      0.101 &    $\circ$ &      0.085 &            &      0.088 &            \\

     srbct &      0.064 &      0.072 &            &      0.074 &            &      0.035 &   $\bullet$ &      0.007 &   $\bullet$ \\
\hline
win-lose-tie &          - &      1-9-0 &            &      1-9-0 &            &      3-7-0 &            &      3-7-0 &            \\
\hline
\end{tabular}

\end{table}

\begin{table}[!]
\scriptsize
\centering
\caption{The error rates of C4.5 applied to the feature subsets selected by GRRF(0.1), GRRF(0.2), RRF(0.9), varSelRF and LASSO logistic regression, respectively. The win-lose-tie results of each competitor compared to GRRF(0.1) with C4.5 are calculated. Here ``$\circ$" or ``$\bullet$" represents a significant difference between a method and GRRF(0.1) with C4.5 at the 0.05 level, according to the paired t-test.
The GRRF methods and the other methods have perform similarly in terms of the accuracy with C4.5. As expected, C4.5 has noticeably higher error rates than RF, shown in Table \ref{table:aveErrRF}. \label{table:aveErrTree}}
\begin{tabular}{cccccccccc}
\hline
           &  GRRF(0.1) &  GRRF(0.2) &            &   RRF(0.9) &            &   varSelRF &            &      LASSO &            \\
           &  -C4.5 &  -C4.5 &            &   -C4.5 &            &   -C4.5 &            &      -C4.5 &            \\
\hline
adenocarcinoma &      0.260 &      0.251 &            &      0.241 &            &      0.248 &            &      0.218 &   $\bullet$ \\

     brain &      0.438 &      0.428 &            &      0.418 &            &      0.394 &   $\bullet$ &      0.476 &    $\circ$ \\

breast.2.class &      0.402 &      0.411 &            &      0.412 &            &      0.413 &            &      0.391 &            \\

breast.3.class &      0.497 &      0.494 &            &      0.491 &            &      0.481 &            &      0.482 &            \\

     colon &      0.274 &      0.262 &            &      0.298 &    $\circ$ &      0.273 &            &      0.242 &   $\bullet$ \\

  leukemia &      0.133 &      0.157 &    $\circ$ &      0.138 &            &      0.146 &            &      0.166 &    $\circ$ \\

  lymphoma &      0.126 &      0.125 &            &      0.126 &            &      0.152 &    $\circ$ &      0.138 &            \\

       nci &      0.626 &      0.636 &            &      0.658 &    $\circ$ &      0.641 &            &      0.646 &            \\

  prostate &      0.157 &      0.138 &   $\bullet$ &      0.154 &            &      0.115 &   $\bullet$ &      0.144 &            \\

     srbct &      0.235 &      0.208 &   $\bullet$ &      0.207 &   $\bullet$ &      0.209 &   $\bullet$ &      0.194 &   $\bullet$ \\
\hline
win-lose-tie &          - &      7-3-0 &            &      5-5-0 &            &      6-4-0 &            &      6-4-0 &            \\
\hline
\end{tabular}
\end{table}

\subsubsection{Feature selection and classification}
First compare the accuracy of RF applied to all the features, denoted as ``All", and the least regularized subset (i.e., features selected by RRF(1)). The number of features and the error rates are shown in Table \ref{table:RFGRRF}. The win-lose-tie results of ``All" compared to RRF(1) are also shown in the tables. To investigate the statistical significance of these results, we applied the paired t-test to the error rates of the two methods from 100 replicates for each data set. The data sets with a significant difference at the 0.05 level are marked with ``$\circ$" or ``$\bullet$" in the table.

The least regularized subset not only has many fewer features than ``All", but also leads to better accuracy  performance on 7 data sets out of 10, and the error rates are significantly different on 4 data sets. Therefore, RRF(1) not only improves interpretability by reducing the number of features, but also can improve the accuracy performance of classification, even for RF, considered as a strong classifier capable of handling irrelevant and relevant variables \citep{houtaoFSReg2012}. It should also be noted that although the least regularized subset is much smaller than the original feature set, the size may still be considered large in some cases, e.g., more than 200 features for the breast.2.class data set. GRRF and RRF with larger regularization, investigated in the following experiments, are able to further reduce the number of features.

Next compare GRRF to RRF and two well-known methods: varSelRF and LASSO logistic regression. The regularization parameter of LASSO logistic regression was selected from \{0.01, 0.02, 0.03, 0.05, 0.1, 0.2\} by 10-fold CV. Here $\gamma\in$ \{0.1, 0.2\} was used for GRRF (i.e., GRRF(0.1) or GRRF(0.2)), and $\lambda$ = 0.9 was used for RRF (i.e., RRF(0.9)). We used a fixed parameter setting for GRRF or RRF, as the parameter sensitivity analysis in the following section shows a consistent trend that GRRF or RRF tends to select more features and also tends to be more accurate, for a smaller $\gamma$ or a larger $\lambda$. We chose these parameters so that a reasonably small number of features can be selected. One can also use cross-validation error to determine an appropriate parameter value customized for each data set and potentially improve these results.


The total number of original features and the average number of features selected by each feature selection method are shown in Table \ref{table:aveFeaN}. All the feature selection methods are able to select a small number of features.

The average error rates of RF applied to all the features (``All") and the subsets selected by different feature selection methods are shown in Table \ref{table:aveErrRF}. GRRF(0.1) with RF outperforms GRRF(0.2) with RF on 9 data sets out of 10, 5 of which have significant differences at the 0.05 level. Even though GRRF(0.1) selects more features than GRRF(0.2), the sizes of the feature subsets are reasonably small (all less than 80 features). 

GRRF(0.1) and RRF(0.9) select a similar number of features. However, GRRF(0.1) with RF outperforms RRF(0.9) with RF on 9 data sets, 5 of which have significant differences at the 0.05 level. Consequently, GRRF selects stronger features than RRF, which is consistent with the simulated experiments. According to the discussion in Section \ref{sec:IssueSparse}, a feature with the maximum Gini information gain in a node with a small number instances may not be truly strong. Yet RRF adds this feature to the subset.

GRRF(0.1) with RF outperforms varSelRF with RF on 7 data sets, 6 of which have significant differences. 
Therefore, GRRF(0.1) may be more favorable than varSelRF for the data sets considered here. Also, as shown in the following section, GRRF has a clear advantage over varSelRF in terms of computational time. Furthermore, GRRF(0.1) with RF outperforms LASSO logistic regression with RF on 7 data sets, 3 of which have significant differences.
The accuracy performance may be improved by applying a logistic regression model to the features selected by LASSO logistic regression. However, tree models like GRRF have a few desirable properties compared to LASSO logistic regression: they can naturally handle mixed categorical and numerical features, and multiple classes, etc. 

The average error rates of C4.5 applied to all the features (``All") and the subsets selected by different feature selection methods are shown in Table \ref{table:aveErrTree}. It can be seen that the error rates of C4.5 are clearly higher than RF shown in Table \ref{table:aveErrRF}. Indeed, RF has been considered as a stronger classifier than C4.5. Interestingly, the differences between the methods in terms of C4.5 are smaller than the RF results. As mentioned by \citep{houtaoFSReg2012}, a relatively weaker classifier is less capable of capturing information from data than a stronger classifier. Consequently, a feature subset that includes strong features, but misses the features of small contributions may not affect the accuracy of C4.5 much, but can affect the accuracy of RF. A weak classifier should only be used for evaluating a feature subset if that classifier is actually used for classification after feature selection. However, if a strong classifier is used for classification after feature selection, or the objective is to evaluate the information contained in the feature subset, a strong classifier should be considered \citep{houtaoFSReg2012}.

\def\myWidth{2.6}
\begin{figure*}[!]
\centering
\subfigure[\label{fig:RRF_fea}]{
\includegraphics[width= \myWidth in]{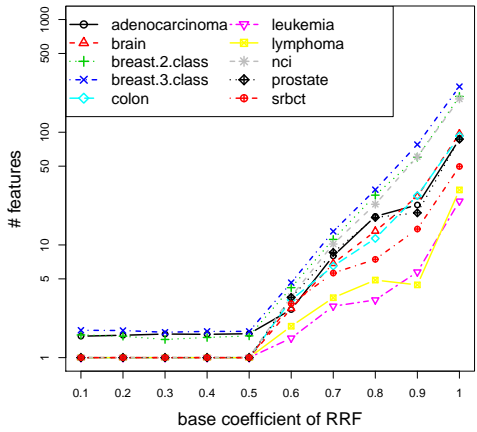}}
\subfigure[\label{fig:GRRF_fea}]{
\includegraphics[width= \myWidth in]{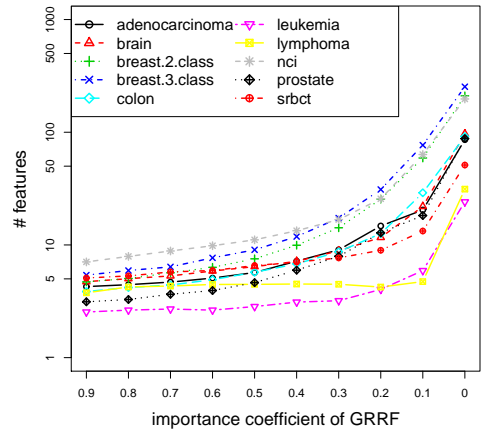}}
\subfigure[\label{fig:RRF_err}]{
\includegraphics[width= \myWidth in]{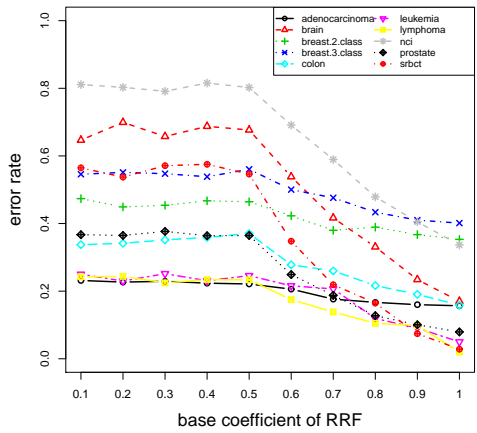}}
\subfigure[\label{fig:GRRF_err}]{
\includegraphics[width= \myWidth in]{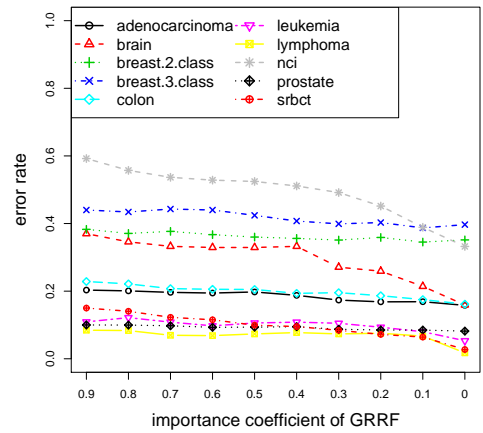}}
\caption[2]{The performance of RRF for selected values of $\lambda$ and GRRF for selected values of $\gamma$. 
Figure \ref{fig:RRF_fea} shows the number of features selected by RRF. A smaller $\lambda$ leads to fewer features.
Figure \ref{fig:GRRF_fea} shows the number of features selected by GRRF. A larger $\gamma$ leads to fewer features.
Figure \ref{fig:RRF_err} shows the error rates of RF applied to the feature subsets selected by RRF for different $\lambda$. The error rates tend to decrease as $\lambda$ increases.
Figure \ref{fig:GRRF_err} shows the error rates of RF applied to the feature subsets selected by GRRF for different $\gamma$.
The error rates tend to decrease as $\gamma$ decreases, but are reasonably robust.
\label{fig:err}}
\end{figure*}


\def\myWidth{3.5}
\begin{figure*}[ht]
\centering
\includegraphics[width= \myWidth in]{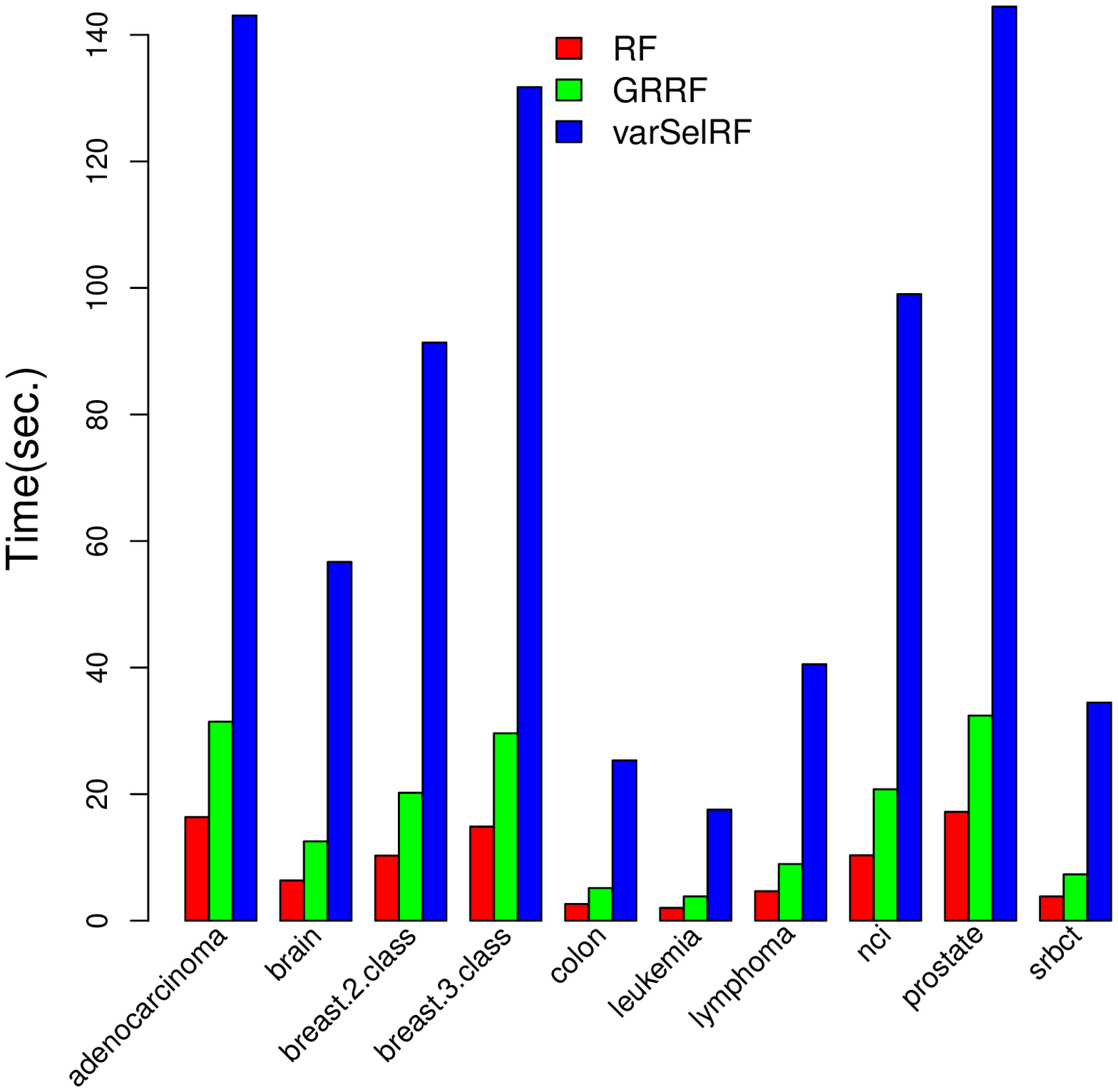}
\caption[2]{Computational time of RF, GRRF and varSelRF. GRRF is more computationally efficient than varSelRF. GRRF takes about twice as much as RF as it builds two ensembles.  \label{fig:timeCompute}}
\end{figure*}

\subsubsection{Parameter Sensitivity and Computational Time}
We investigated the performance of RRF and GRRF with different parameter settings: $\lambda\in$ \{0.1, 0.2, 0.3, 0.4, 0.5, 0.6, 0.7, 0.8, 0.9, 1\} for RRF and $\gamma\in\{0.9, 0.8, 0.7, 0.6, 0.5, 0.4, 0.3, 0.2, 0.1, 0\}$ for GRRF. The parameter values are arranged by the degree of regularization in a decreasing order for both methods. 


The sizes of the feature subsets, averaged over 100 replicates, for each parameter setting of RRF and GRRF, are shown in Figures \ref{fig:RRF_fea} and \ref{fig:GRRF_fea}, respectively. The number of features tends to increase as $\lambda$ increases for RRF, or as $\gamma$ decreases for GRRF.
The consistent trend illustrates that one can control the size of the feature subset by adjusting the parameters. 

The error rates, averaged over 100 replicates, for each parameter setting of RRF and GRRF are shown in Figures \ref{fig:RRF_err} and \ref{fig:GRRF_err}, respectively. In general, the error rates tend to decrease as $\lambda$ increases for RRF, or as $\gamma$ decreases for GRRF. However, for most data sets, the error rates of GRRF seem to be reasonably robust to the changes of $\gamma$.
As mentioned, RRF(1) and GRRF(0) are equivalent, and, therefore, they have similar number of features and error rates for every data set (differing only by random selections).

The computational time of the RF-based methods: RF, GRRF and varSelRF is shown in Figure \ref{fig:timeCompute}. As expected, RF was fast for training on these data sets. GRRF builds two ensemble models, and, thus, the computational time of GRRF is only about twice as much as RF. However, varSelRF needs to build multiple ensemble models, and the computational advantage of GRRF over varSelRF is clear.

\section{Conclusions}\label{sec:CON}
We derive an upper bound for the number of distinct Gini information gain values in a tree node for binary classification problems. The upper bound indicates that the Gini information gain may not be able to distinguish features at nodes with a small number of instances, which poses a challenge for RRF that selects a feature only using the instances at a node. Motivated by this node sparsity issue, we propose an enhanced method called the guided, regularized random forest (GRRF), in which a preliminary random forest is used to generate the initial variable importance scores to guide the regularized feature selection process of RRF. The importance scores from the preliminary RF help GRRF select better features when many features share the same maximal Gini information gain at a node. For the experiments here, GRRF is more favorable than RRF, computationally efficient, selects a small set of features, and has competitive accuracy performance.

Feature selection eliminates irrelevant or redundant features, but also may eliminate features of small importance. This may not affect the performance of a weak classifier which is less capable of capturing small information, but may affect the performance of a strong classifier such as RF \citep{houtaoFSReg2012}. Still, we found that the least regularized subset selected by RRF with the minimal regularization produces better accuracy performance than the complete feature set.

Finally we note that although RRF and GRRF can be used as classifiers, they are designed for feature selection. The trees in RRF and GRRF are not built independently as the features selected in previous trees have an impact on the trees built later. Therefore, as a classifier, RRF or GRRF may have a higher variance than RF because the trees are correlated. Consequently, in this work we applied RF on the feature subset selected by GRRF or RRF for classification.



\bibliographystyle{model1b-num-names}
\bibliography{RRF}

\end{document}